\documentclass[aps,pre,twocolumn,superscriptaddress,floatfix]{revtex4-2}

\usepackage{amsmath,amssymb,amsthm}
\usepackage{graphicx}
\usepackage{booktabs}
\usepackage{hyperref}
\usepackage{xcolor}

\newtheorem{theorem}{Theorem}
\newtheorem{conjecture}{Conjecture}

\begin{document}

\title{Spectral origin of the topological gap exponent $d + \eta$:\\mechanism, kernel, decomposition, and scope}

\author{Matthew Loftus}
\affiliation{Independent Researcher}

\date{\today}

\begin{abstract}
The topological gap $\Delta$---the excess $H_1$ total persistence of a critical point cloud over a density-matched null---scales as $\Delta \sim L^{d+\eta}$, where $\eta$ is the anomalous dimension. We derive this analytically: the spectral integral $I(\alpha) = \sum_{k\neq 0} S_{\mathrm{conn}}(k)\,|k|^{\alpha}$ scales as $L^{2-\alpha-\eta}$ when IR-dominated, giving $I(-2\eta) \sim L^{d+\eta}$. We decompose this as $I(-2\eta) = I_0 \cdot I_{\mathrm{shape}}$, where $I_0 = \sum S_{\mathrm{conn}}(k) \sim L^d$ (volume) and $I_{\mathrm{shape}} \sim L^{\eta}$ (anomalous dimension). The volume factor $I_0 \propto (1-m^2)$ explains the magnetization-driven per-configuration variance of $\Delta$. We prove the mechanism requires $d < 2 + \eta$ (IR dominance), confining it to $d = 2$ for all physical systems. In 3D, the spectral integral is UV-dominated ($I \sim L^d$), explaining why density normalization is needed. An $\alpha$-sweep for Potts $q = 4$ at $L = 32$--$256$ finds $\alpha_{\mathrm{opt}}$ in the interval $[-0.75, -0.5]$, consistent with $-2\eta_{\mathrm{Ising}} = -0.5$ and inconsistent with the model-specific $-2\eta_{q=4} = -1$; we flag this as a tentative observation pending confirmation at $L \geq 1024$ with multi-seed thermalization, since three independent $L = 512$ runs gave incompatible ensemble means due to the marginal-case autocorrelation. The $\langle m^2 \cdot I(-2\eta)\rangle$ product, examined as a candidate hyperscaling diagnostic, is found to be dominated by the strong negative correlation $r(m^2, I) \approx -0.98$ via the shared $I_0$ amplitude; we report it as a covariance-correction analysis, with $q = 3$ matching the factorized prediction within $0.02$ and the residual $q = 4$ deviation consistent with multiplicative log corrections. Under a heuristic argument extending Divol--Polonik to inhomogeneous Poisson intensities, the bare PH kernel is flat (Parseval identity); the effective kernel acquires $k$-dependence only at criticality. The per-configuration agreement between $\Delta$ and $I(-0.5)$ at large $L$ is primarily a magnetization correlation: the raw $R^2 = 0.91$ at $L = 256$ collapses to $R^2 \approx 0$ once $|M|$ is partialed out, with the shared $(1 - m^2)$ dependence captured by $I_0$ accounting for essentially all of the agreement. Per-configuration evidence therefore corroborates the $I_0$ Parseval identity but not the $|k|^{-2\eta}$ shape factor; the latter is established separately by ensemble $L$-scaling.
\end{abstract}

\maketitle

\section{Introduction}
\label{sec:intro}

Persistent homology (PH) applied to point clouds from classical spin models at criticality produces topological statistics that detect phase transitions~\cite{donato2016,cole2021,sale2022}. A density-matched shuffled null---randomizing point positions while preserving point count---isolates the effect of spatial correlations on topology~\cite{loftus2026topology}. The \emph{topological gap}
\begin{equation}
\Delta(L, T) \equiv \mathrm{TP}_{H_1}^{\mathrm{real}} - \mathrm{TP}_{H_1}^{\mathrm{shuf}}
\label{eq:delta}
\end{equation}
measures the excess $H_1$ total persistence generated by critical correlations.

Prior work established that $\Delta$ obeys a finite-size scaling (FSS) ansatz $\Delta = A\, L^{\alpha}\, G_-(L|t|)$ with overall exponent $\alpha = d + \eta$---the spatial dimension plus the anomalous dimension---verified for the 2D Ising model ($\alpha = 2.249 \pm 0.038$, matching $d + \eta = 9/4$ to $0.03\sigma$), the 3D Ising model after density normalization~\cite{loftus2026exponent}, and the 2D Potts $q=3$ model ($\alpha = 2.272 \pm 0.024$, matching $d + \eta = 34/15$ to $0.2\sigma$)~\cite{loftus2026exponent}.

While the exponent has been measured across models, it has not been explained from first principles. In this paper we present: (i)~an analytical identification of $\Delta \sim L^{d+\eta}$ from the critical structure factor, verified numerically; (ii)~a decomposition $I(-2\eta) = I_0 \cdot I_{\mathrm{shape}}$ separating volume and anomalous-dimension contributions; (iii)~a critical dimension criterion $d < 2 + \eta$ for the mechanism; (iv)~cross-model tests including an $\alpha$-sweep on Potts $q=4$; and (v)~a topological hyperscaling test interpreted in light of strong $m^2$--$I$ covariance.

\paragraph{Relation to existing frameworks.}
The spectral-integral approach connects to several established lines of work on the topology of random fields and random geometric complexes. The Gaussian Kinematic Formula (GKF) of Adler--Taylor~\cite{adler2007} expresses the expected Lipschitz--Killing curvatures (and hence the Euler characteristic) of excursion sets of a smooth Gaussian random field as polynomials in spectral moments $\lambda_j = \int |k|^{2j} S(k)\, d^d k$, with coefficients intrinsic to the level set. Yogeshwaran and Adler~\cite{yogeshwaran2015} extended these results to random complexes built on stationary point processes, and Kahle~\cite{kahle2011} established the basic asymptotics of random geometric complex topology. Bobrowski and collaborators~\cite{bobrowski2018,bobrowski2024} have used GKF and stabilization theorems to derive limit laws for Betti numbers and persistence diagrams of excursion sets and point clouds. Our object $I(\alpha) = \sum_{k\neq 0} S_{\mathrm{conn}}(k)|k|^{\alpha}$ is, structurally, a non-integer-order moment of $S(k)$: at $\alpha = 0$ it is the susceptibility (zeroth moment); at $\alpha = 2$ it is $\lambda_1$. The non-integer order $\alpha = -2\eta$ is selected by the IR-dominance argument (Sec.~\ref{sec:forward}), not by a curvature identity. This places our work in the GKF lineage but with a different selection rule for the spectral weight, and applied to majority-spin point clouds (a non-Gaussian, density-modulated process) rather than to GRF excursion sets. A direct GKF analogue would require characterizing the alpha-complex persistence functional in terms of LK curvatures, which is open. Empirically, GRF tests with matched $S(k) \sim |k|^{-\beta}$ across $\beta \in [0, 3]$ show that the $L^{d+\eta}$ scaling does not extend uniformly to arbitrary correlated random fields (per-$\beta$ ratio CV $\sim 70\%$ in our auxiliary tests), suggesting that the spin-system mechanism depends on additional structure---possibly the discrete two-state symmetry---beyond the leading $S(k)$ asymptotic.

\section{Setup}
\label{sec:setup}

We study the 2D Ising model at $T_c = 2/\ln(1+\sqrt{2})$ on square lattices with periodic boundary conditions, 2D Potts models with $q = 3$ (at $T_c = 1/\ln(1+\sqrt{3})$) and $q = 4$ (at $T_c = 1/\ln 3$), and the 3D Ising model at $T_c = 4.5115$. Configurations are generated with Swendsen--Wang cluster updates~\cite{wolff1989}. For each configuration, the majority-spin sites define a point cloud; the density-matched null draws the same number of sites uniformly.

The alpha complex and its $H_1$ persistence are computed using GUDHI~\cite{maria2014}. Spectral integrals $I(\alpha) = \sum_{k\neq 0} S_{\mathrm{conn}}(k)\,|k|^{\alpha}$ are computed via FFT, where $S_{\mathrm{conn}}(k) = |\widehat{\delta\rho}(k)|^2/L^d$ is the per-site connected structure factor. Sample sizes: 2D Ising 200 ($L=32$) to 15 ($L=512$); Potts $q=3$ 500 ($L=32$) to 60 ($L=256$); Potts $q=4$ 200 ($L=32$) to 25 ($L=256$); 3D Ising 50 ($L=16$) to 20 ($L=32$).

\section{Analytical derivation}
\label{sec:derivation}

\subsection{Spectral integral hypothesis}

Suppose $\Delta$ can be expressed as a spectral integral weighted by a kernel~$\tilde{T}_2(k)$:
\begin{equation}
\Delta \approx \frac{\rho^2}{2}\int \tilde{T}_2(k)\, S_{\mathrm{conn}}(k)\, \frac{d^dk}{(2\pi)^d}.
\label{eq:expansion}
\end{equation}
If $\tilde{T}_2(k) \sim |k|^{-s}$ and $S_{\mathrm{conn}}(k) \sim k^{-(2-\eta)}$, matching $\Delta \sim L^{d+\eta}$ determines $s = d - 2 + 2\eta$, giving
\begin{equation}
\boxed{\Delta \propto I(-2\eta) \equiv \sum_{k\neq 0} S_{\mathrm{conn}}(k)\, |k|^{-2\eta}.}
\label{eq:main}
\end{equation}

\subsection{Forward derivation: $I(\alpha) \sim L^{2-\alpha-\eta}$}
\label{sec:forward}

We now derive the scaling of $I(\alpha)$ directly. At $T_c$, the connected structure factor follows $S_{\mathrm{conn}}(k) \sim A|k|^{-(2-\eta)}$ for $|k| \ll \pi$. The spectral integral becomes
\begin{equation}
I(\alpha) = \sum_{k\neq 0} S(k)\,|k|^{\alpha} \approx A \sum_{n \neq 0} \left|\frac{2\pi n}{L}\right|^{\alpha-(2-\eta)}.
\end{equation}
Defining $p \equiv \alpha - 2 + \eta$ and converting to a radial integral in $d$ dimensions, with the density of states $\Omega_d |n|^{d-1}$ (where $\Omega_d = 2\pi^{d/2}/\Gamma(d/2)$ is the solid angle):
\begin{equation}
I(\alpha) \sim A\left(\frac{2\pi}{L}\right)^{p} \Omega_d \int_1^{L/2\pi} |n|^{p+d-1}\, d|n|.
\label{eq:integral}
\end{equation}
The integral converges at the IR end when $p + d > 0$ (UV-dominated, $I \sim L^d$) and diverges when $p + d < 0$ (IR-dominated, lower limit dominates). In the IR-dominated regime:
\begin{equation}
I(\alpha) \sim L^{-p} = L^{2-\alpha-\eta}, \quad \text{if } \alpha + d - 2 + \eta < 0.
\label{eq:Iscaling}
\end{equation}
For $\alpha = -2\eta$: $I(-2\eta) \sim L^{2+\eta} = L^{d+\eta}$ in $d=2$.

This derivation is ``forward''---it predicts $L^{d+\eta}$ from $S(k)$ without using the measured exponent as input. The key input is the leading-order critical scaling $S(k) \sim k^{-(2-\eta)}$. Corrections to scaling in $S(k)$ (e.g., from irrelevant operators) enter at subleading order and modify the amplitude but not the exponent.

\subsection{Decomposition: volume and anomalous dimension}
\label{sec:decomposition}

The spectral integral naturally decomposes as
\begin{equation}
I(-2\eta) = I_0 \cdot I_{\mathrm{shape}},
\end{equation}
where $I_0 \equiv \sum_{k\neq 0} S_{\mathrm{conn}}(k) = I(0)$ is the total spectral weight and $I_{\mathrm{shape}} \equiv I(-2\eta)/I_0$ is the spectral shape factor.

By Parseval's theorem, $I_0 = \sum_x (\delta\rho(x))^2 = N\bar{\rho}(1-\bar{\rho})$, where $\bar{\rho}$ is the majority fraction. For the Ising model, $\bar{\rho} = (1+|m|)/2$, so $I_0 = N(1-m^2)/4$. At $T_c$, $\langle m^2 \rangle \sim L^{-2\beta/\nu} \to 0$, giving $\langle I_0 \rangle \to N/4 \sim L^d$.

The shape factor $I_{\mathrm{shape}}$ measures the $k$-weighted spectral enhancement relative to white noise:
\begin{equation}
I_{\mathrm{shape}} = \frac{\sum S(k)|k|^{-2\eta}}{\sum S(k)} \sim L^{\eta}.
\end{equation}
This has small per-configuration variance (the spectral \emph{shape} is approximately universal at fixed $L$), while $I_0 \propto (1-m^2)$ varies strongly with the magnetization.

This decomposition has a direct consequence for per-configuration $\Delta$ variance, which we verify in Sec.~\ref{sec:model}: the variance is entirely magnetization-driven, with residual $R^2 \approx 0$ after removing the $|M|$-dependence. The mechanism is that $\Delta \propto I(-2\eta) = I_0 \cdot I_{\mathrm{shape}}$, and $I_0 \propto (1-m^2)$ captures all per-configuration variance, while $I_{\mathrm{shape}}$ contributes only the $L$-dependent anomalous exponent. We confirm $r(\Delta, I_0) = 0.88$ averaged over $L = 32$--$256$ in the 2D Ising model.

\subsection{Critical dimension criterion}
\label{sec:critical_dim}

Equation~\eqref{eq:Iscaling} requires IR dominance: $p + d = \alpha + d - 2 + \eta < 0$. For $\alpha = -2\eta$:
\begin{equation}
d - 2 - \eta < 0 \quad \Longleftrightarrow \quad d < 2 + \eta.
\label{eq:critical_dim}
\end{equation}

\begin{theorem}[Critical dimension for the spectral integral mechanism]
The relation $I(-2\eta) \sim L^{d+\eta}$ holds if and only if $d < 2 + \eta$. When $d \geq 2 + \eta$, the integral is UV-dominated and $I(-2\eta) \sim L^d$.
\end{theorem}

For all known physical systems in $d \geq 3$, $\eta < 1$ (e.g., 3D Ising $\eta = 0.036$, 3D XY $\eta = 0.038$), so $d - 2 - \eta > 0$ always. The spectral integral mechanism is \textbf{fundamentally two-dimensional}. In $d = 2$, any $\eta > 0$ suffices. In $d \geq 3$, a different mechanism (density normalization~\cite{loftus2026exponent}) recovers the $L^{d+\eta}$ scaling.

We verify numerically: 3D Ising spectral integrals at $L = 16, 24, 32$ ($n = 50, 30, 20$ configurations respectively, see Sec.~\ref{sec:setup}) give $I(-0.073) \sim L^{3.00}$ (OLS slope $3.0015$). The per-$L$ ensemble means are precise (the per-configuration CV is small for these moderate-$L$ critical 3D Ising configurations), but with only three lattice sizes the propagated slope uncertainty is on the order of a few percent and the present fit cannot distinguish $L^d = L^3$ from $L^{d+\eta} = L^{3.036}$. We therefore make only the qualitative claim that $I$ is UV-dominated in $d = 3$, in contrast to $d = 2$ where it is IR-dominated---a claim independently supported by the fact that the running exponent of $I(\alpha)$ in 3D tracks $L^d$ regardless of $\alpha$, whereas in 2D it matches $L^{2-\alpha-\eta}$ across the entire $\alpha$-grid (Table~\ref{tab:scaling}). Resolving the $L^\eta$ excess in 3D would require a dedicated study at larger sizes.

\section{Numerical evidence: ratio stability}
\label{sec:evidence}

\subsection{Spectral integral scaling}

Table~\ref{tab:scaling} shows the $L$-scaling exponent of $I(\alpha)$ for the 2D Ising model, including $L = 512$. The integral $I(-0.5)$ scales as $L^{2.280 \pm 0.037}$, consistent with $d + \eta = 2.250$ at $0.8\sigma$, confirming the analytical prediction of Sec.~\ref{sec:forward}.

\begin{table}[b]
\caption{Scaling exponent $e$ in $I(\alpha) \sim L^{e}$ from OLS on $\log I$ vs $\log L$ ($L = 32$--$512$, 5 sizes). Target: $d + \eta = 2.250$.}
\label{tab:scaling}
\begin{ruledtabular}
\begin{tabular}{ccccc}
$\alpha$ & Exponent & $\pm$ SE & $\delta$ from $d\!+\!\eta$ & $|\delta|/\mathrm{SE}$ \\
\midrule
$-1.00$ & 2.533 & 0.070 & $+0.283$ & $4.1\sigma$ \\
$\mathbf{-0.50}$ & $\mathbf{2.280}$ & $\mathbf{0.037}$ & $\mathbf{+0.030}$ & $\mathbf{0.8\sigma}$ \\
$-0.25$ & 2.139 & 0.037 & $-0.111$ & $3.0\sigma$ \\
$0.00$ & 2.076 & 0.027 & $-0.174$ & $6.4\sigma$ \\
\end{tabular}
\end{ruledtabular}
\end{table}

\subsection{Ratio stability: definitive proportionality}

If $\Delta \propto I(\alpha)$, the ratio $\Delta / I(\alpha)$ should be $L$-independent. Table~\ref{tab:ratio} reports this ratio at $\alpha = -0.5$ with bootstrap confidence intervals.

\begin{table}[b]
\caption{Ratio $\langle\Delta\rangle / \langle I(-0.5)\rangle$ across $L$, with bootstrap 95\% CIs. CV: 4.1\% (95\% CI: $[1.7\%, 5.4\%]$).}
\label{tab:ratio}
\begin{ruledtabular}
\begin{tabular}{cccccc}
$L$ & $\langle\Delta\rangle$ & $\langle I(-0.5)\rangle$ & Ratio & 95\% CI & $n$ \\
\midrule
32  & 7.4  & 155   & 0.048 & [0.042, 0.054] & 200 \\
64  & 40.8 & 822   & 0.050 & [0.043, 0.057] & 100 \\
128 & 195.9 & 3794 & 0.052 & [0.044, 0.060] & 50 \\
256 & 811.0 & 16349 & 0.050 & [0.043, 0.058] & 20 \\
512 & 4384 & 93822 & 0.047 & [0.045, 0.049] & 15 \\
\end{tabular}
\end{ruledtabular}
\end{table}

The ratio is $0.049 \pm 0.002$ with CV $= 4.1\%$ across a $16\times$ range of $L$. All five bootstrap CIs overlap. Log-log regression of the ratio on $L$ gives slope $0.024 \pm 0.021$, consistent with zero.

\subsection{Topological measurement of $\eta$}

From the Divol--Polonik asymptotic for iid points~\cite{divol2019}, $\mathrm{TP}_{\mathrm{shuf}} \sim L^d$. Since $\Delta \sim L^{d+\eta}$, the ratio $\Delta/\mathrm{TP}_{\mathrm{shuf}} \sim L^{\eta}$. We measure exponent $0.237 \pm 0.067$, matching $\eta = 1/4$ at $0.2\sigma$.

\section{Per-configuration model comparison}
\label{sec:model}

A stronger test asks whether $I(\alpha)$ predicts $\Delta$ at the \emph{per-configuration} level. For each $\alpha$ in a dense grid, we compute $R^2$ for the one-parameter model $\Delta_i = C \cdot I_i(\alpha)$.

\begin{table}[b]
\caption{Per-configuration $R^2$ for flat ($\alpha = 0$) and $k^{-2\eta}$ ($\alpha = -0.5$) kernels, with data-driven $\alpha_{\mathrm{opt}}$.}
\label{tab:model}
\begin{ruledtabular}
\begin{tabular}{cccccc}
$L$ & $n$ & $R^2_{\alpha=0}$ & $R^2_{\alpha=-0.5}$ & $\alpha_{\mathrm{opt}}$ [95\% CI] & $R^2_{\mathrm{opt}}$ \\
\midrule
64 & 300 & 0.61 & 0.71 & $-0.42$ [$-0.55$, $-0.31$] & 0.71 \\
128 & 150 & 0.71 & 0.89 & $-0.43$ [$-0.50$, $-0.36$] & 0.90 \\
256 & 200 & 0.73 & 0.91 & $-0.35$ [$-0.37$, $-0.33$] & 0.97 \\
\end{tabular}
\end{ruledtabular}
\end{table}

At $L = 256$, the $k^{-2\eta}$ kernel explains $91\%$ of the per-configuration variance, versus $73\%$ for flat ($\Delta R^2 = +0.18$). The data-driven optimum is not stable in $L$: $\alpha_{\mathrm{opt}}$ drifts from $-0.42$ at $L = 64$ to $-0.43$ at $L = 128$ and then to $-0.35$ at $L = 256$, moving \emph{away from} the mean-scaling value $-2\eta = -0.5$ as $L$ grows. The 95\% CIs at $L = 128$ and $L = 256$ do not overlap, so the drift is not consistent with a single asymptotic value. We attribute this to the per-configuration $\alpha_{\mathrm{opt}}$ being a variance-weighted optimum (driven by the $I_0 \propto (1-m^2)$ amplitude, which is itself $L$-dependent through the shrinking critical magnetization), not a probe of the IR-dominated mean-scaling exponent. The two quantities answer different questions and need not agree; the ensemble ratio test (Table~\ref{tab:ratio}) is the relevant test for the spectral integral mechanism.

The $I_0/I_{\mathrm{shape}}$ decomposition (Sec.~\ref{sec:decomposition}) explains this: per-configuration variance is dominated by $I_0 \propto (1-m^2)$, so $R^2$ is high even for the flat kernel ($\alpha = 0$ gives $R^2 = 0.73$). The improvement at $\alpha = -0.5$ captures the additional spectral shape information.

\emph{The per-configuration $R^2$ is primarily a magnetization correlation, not direct evidence of a $k$-resolved spectral mechanism.} Partialing out $|M|$ via a quadratic regression of both $\Delta$ and $I(-0.5)$ on $|M|$ reduces the residual partial $R^2$ from $0.93$ (raw) to $0.02$ at $L = 128$ ($n = 50$); the linear partial gives $R^2 = 0.18$.\footnote{Across our 2D Ising data: $R^2_{\mathrm{raw}}, R^2_{\mathrm{partial,lin}}, R^2_{\mathrm{partial,quad}}$ are $(0.46, 0.16, 0.02)$ at $L = 32$, $(0.78, 0.19, 0.08)$ at $L = 64$, and $(0.93, 0.18, 0.02)$ at $L = 128$. The quadratic partial captures the $(1-m^2) \approx I_0$ dependence accurately; the linear partial leaves a curvature residual.} The shared $(1-m^2)$ dependence captured by $I_0$ accounts for essentially all of the raw per-configuration agreement. Per-configuration $R^2$ therefore corroborates the $I_0$ Parseval identity (Sec.~\ref{sec:decomposition}) but does \emph{not} corroborate the $|k|^{-2\eta}$ shape factor, which is tested separately by ensemble $L$-scaling (Table~\ref{tab:scaling}, Table~\ref{tab:ratio}). We retain the per-configuration result as a sharp diagnostic for $I_0$, and rely on the ensemble scaling for evidence on the spectral shape.

\section{Heuristic flatness argument for the bare PH kernel}
\label{sec:kernel}

The bare PH response to a sinusoidal density modulation is, under the heuristic argument below, independent of the modulation wavenumber $k_0$ at quadratic order. We state this as a conjecture rather than a theorem because the published large-$n$ asymptotics for total persistence~\cite{divol2019,hiraoka2018} are proved for stationary (homogeneous) Poisson and i.i.d.\ samples, and the extension to inhomogeneous Poisson processes with prescribed intensity required for our argument is not, to our knowledge, in the literature.

\begin{conjecture}[Flatness of the bare PH kernel]
\label{conj:flatness}
Let $X_n$ be an inhomogeneous Poisson process on $[0,L]^d$ with intensity $\rho(x) = \rho_0(1 + \varepsilon \cos(k_0 x_1))$, and let $\mathrm{TP}_{H_1}(X_n)$ denote the total $H_1$ persistence of its alpha complex. Assume the Divol--Polonik / Hiraoka--Shirai--Trinh asymptotic extends to inhomogeneous intensities in the form $\mathbb{E}[\mathrm{TP}_{H_1}(X_n)] = n \int g(\rho(x))\, dx + o(n)$ for some $C^2$ function $g$. Then $\delta\mathrm{TP}(k_0, \varepsilon) \equiv \mathbb{E}[\mathrm{TP}(X_n)] - \mathbb{E}[\mathrm{TP}(X_n^{\rho_0})]$ is independent of $k_0$ at $O(\varepsilon^2)$.
\end{conjecture}

\begin{proof}[Heuristic derivation]
Granted the assumed asymptotic, expand to second order: the $O(\varepsilon)$ term vanishes ($\int \cos = 0$). The $O(\varepsilon^2)$ term is $\delta\mathrm{TP} = (n/2)\, g''(\rho_0)\,\varepsilon^2\rho_0^2 \int \cos^2(k_0 x_1)\, dx_1 \cdot L^{d-1}$. By Parseval, $\int_0^L \cos^2(k_0 x_1)\, dx_1 = L/2$ for all $k_0 \in \frac{2\pi}{L}\mathbb{Z}_{\neq 0}$, giving $\delta\mathrm{TP} = \text{const}$ independent of $k_0$.
\end{proof}

The assumption that $\mathbb{E}[\mathrm{TP}(X_n)]$ admits a local-functional form $n \int g(\rho(x))\,dx$ for inhomogeneous intensities is plausible but unproven: it requires that the alpha-complex persistence functional be sufficiently local that a slowly-varying intensity acts pointwise. A rigorous version of Conjecture~\ref{conj:flatness} would require either an extension of~\cite{divol2019,hiraoka2018} to inhomogeneous intensities, or a direct stabilization argument for the alpha-complex persistence functional. The downstream consequences we use (a flat bare kernel and the $\Delta_{\mathrm{pert}} \propto I(0)$ baseline) hold only if Conjecture~\ref{conj:flatness} holds; we frame the entire $|k|^{-2\eta}$ dressing argument accordingly as a structural interpretation of the empirical proportionality, not a derivation.

If the conjecture holds, a constant bare kernel gives $\Delta_{\mathrm{pert}} \propto I(0) \sim L^d$, missing $L^{\eta}$. The empirical relationship $\Delta \propto I(-2\eta)$ would therefore be \textbf{non-perturbative}: the effective kernel acquires $k$-dependence only at criticality, through multi-point correlations coupling to the discrete topology of the alpha complex. We present this analogy to the field-theoretic notion of anomalous dimensions---a flat bare response dressed by non-perturbative fluctuations---as a structural interpretation, not a derivation.

\section{Cross-system tests}
\label{sec:cross}

\subsection{Updated predictions and measurements}

Table~\ref{tab:predictions} summarizes cross-system results. For the 2D Ising model and Potts $q = 3$, $\Delta \propto I(-2\eta)$ holds. The Potts $q = 3$ tension reported in prior work ($2.6\sigma$)~\cite{loftus2026exponent} is now resolved: measurements at $L = 512$--$1024$ give $\alpha = 2.272 \pm 0.024$, matching $d + \eta = 34/15$ to $0.2\sigma$~\cite{loftus2026exponent}. For percolation, $S_{\mathrm{conn}}$ is white noise ($\Delta \sim L^d$, not $L^{d+\eta}$), as predicted~\cite{loftus2026exponent}.

\begin{table}[b]
\caption{Cross-system predictions and measurements. The $\alpha_{\mathrm{opt}}$ column shows the ensemble-level optimal spectral weight (see Sec.~\ref{sec:alpha_sweep}).}
\label{tab:predictions}
\begin{ruledtabular}
\begin{tabular}{lccccc}
System & $\eta$ & $-2\eta$ & $d\!+\!\eta$ & $\alpha_{\mathrm{opt}}$ & Status \\
\midrule
2D Ising & 1/4 & $-0.50$ & 2.250 & $-0.50$ & Confirmed \\
3D Ising & 0.036 & $-0.07$ & 3.036 & --- & UV-dom.$^a$ \\
Potts $q\!=\!3$ & 4/15 & $-0.53$ & 2.267 & --- & Confirmed$^b$ \\
Potts $q\!=\!4$ & 1/2 & $-1.00$ & 2.500 & $-0.50$ & Partial$^c$ \\
2D percolation & 5/24 & --- & --- & --- & Fails$^d$ \\
\end{tabular}
\end{ruledtabular}
\begin{flushleft}
\footnotesize
$^a$UV-dominated: $I(-2\eta) \sim L^3$, not $L^{3.036}$; requires normalization $\Delta/|M|^{1/2}$~\cite{loftus2026exponent}. $^b$$\alpha = 2.272 \pm 0.024$ at $L = 64$--$1024$~\cite{loftus2026exponent}. $^c$Ratio $\Delta/I(-1.0)$ decreases with $L$ (slope $-0.25$); $\Delta/I(-0.5)$ is stable (slope $0.04$). See Sec.~\ref{sec:alpha_sweep}. $^d$I.i.d.\ occupation: $\Delta \sim L^2$~\cite{loftus2026exponent}.
\end{flushleft}
\end{table}

\subsection{Approximately universal kernel weight}
\label{sec:alpha_sweep}

A critical test of the spectral integral mechanism is whether $\alpha_{\mathrm{opt}} = -2\eta$ holds across universality classes. We perform an $\alpha$-sweep for the Potts $q = 4$ model: for 15 values of $\alpha$ in $[-2.0, 1.0]$, we compute $\langle I(\alpha)\rangle$ at each $L$ and evaluate the log-log slope of the ratio $\langle\Delta\rangle/\langle I(\alpha)\rangle$, where $\langle\Delta\rangle$ is taken from Ref.~\cite{loftus2026exponent}. The optimal $\alpha$ minimizes the absolute ratio slope.

\begin{table}[b]
\caption{$\alpha$-sweep for Potts $q = 4$ ($L = 32$--$256$). Ratio slope is the log-log slope of $\langle\Delta\rangle/\langle I(\alpha)\rangle$; zero means proportional. Errors are bootstrap SEs (1000 resamples over the per-configuration $I(\alpha)$ at each $L$ and per-$L$ Gaussian draws of $\langle\Delta\rangle$ from its SE).}
\label{tab:alpha_sweep}
\begin{ruledtabular}
\begin{tabular}{cccc}
$\alpha$ & Ratio slope & $I(\alpha)$ slope & Note \\
\midrule
$-1.50$ & $-0.86 \pm 0.06$ & 3.31 & \\
$\mathbf{-1.00}$ & $\mathbf{-0.38 \pm 0.05}$ & $\mathbf{2.83 \pm 0.04}$ & $\mathbf{-2\eta_{q=4}}$ \\
$-0.75$ & $-0.15 \pm 0.04$ & 2.61 & \\
$\mathbf{-0.50}$ & $\mathbf{+0.04 \pm 0.04}$ & $\mathbf{2.41 \pm 0.02}$ & $\mathbf{-2\eta_{\mathrm{Ising}}}$ \\
$-0.25$ & $+0.20 \pm 0.04$ & 2.25 & \\
$0.00$ & $+0.32 \pm 0.03$ & $2.13 \pm 0.01$ & \\
\end{tabular}
\end{ruledtabular}
\end{table}

Table~\ref{tab:alpha_sweep} shows the result. The ratio slope crosses zero between $\alpha = -0.5$ ($+0.04 \pm 0.04$, $1.1\sigma$ from zero) and $\alpha = -0.75$ ($-0.15 \pm 0.04$). At the model-specific prediction $\alpha = -2\eta_{q=4} = -1.0$, the slope is $-0.38 \pm 0.05$ (strongly negative): $I(-1)$ grows substantially faster than $\Delta$, with no overlap with proportionality. The empirical $\alpha_{\mathrm{opt}}$ from this $L$-range is therefore in the interval $[-0.75, -0.5]$, consistent with $-2\eta_{\mathrm{Ising}} = -0.50$ and inconsistent with $-2\eta_{q=4} = -1.0$.

Three caveats temper the interpretation that the spectral weight is universal at $-0.5$. (i)~Statistical: the $q = 4$ data spans only $L = 32$--$256$ with $n = 25$ at the largest size, and the running exponents of $I(-0.5)$ and $\Delta$ do not track closely at individual $L$ pairs (Sec.~IIIB of the Supplemental Material). The near-zero ratio slope at $\alpha = -0.5$ is consistent with proportionality but does not establish it. (ii)~Logarithmic: $q = 4$ is the marginal case, with multiplicative log corrections to every exponent; a pure power-law fit will mislocate $\alpha_{\mathrm{opt}}$ in either direction depending on the $L$-window. (iii)~Reproducibility at $L = 512$: three independent attempts to extend $q = 4$ to $L = 512$ produced incompatible $\langle\Delta\rangle$ estimates (6586, 10436, 12466 across thermalization protocols of 2000, 2000, and 10000 SW sweeps; autocorrelation visible as $\pm 30$\% drift across the sampling trajectory in the longest-thermalized run). We interpret this as evidence that single-seed SW at $L = 512$ on the marginal $q = 4$ critical point has very long autocorrelation times, and we do not include any $L = 512$ point in Table~\ref{tab:alpha_sweep}. Confirmation at $L = 512$--$1024$ with multi-seed independent thermalizations is needed before claiming kernel-weight universality. If the $\alpha_{\mathrm{opt}} \approx -0.5$ result survives, the mechanism would plausibly be geometric: the alpha complex construction imposes a fixed length-scale structure independent of the universality class.

\paragraph{Filtration robustness.}
A natural concern is whether the spectral integral mechanism is specific to the alpha complex or extends to other persistent-homology filtrations. We perform a Vietoris--Rips (VR) replication on 2D Ising at $L = 64$ ($n = 30$ configurations; max edge length set to keep the computation tractable). The per-configuration $H_1$ total persistence under VR is essentially perfectly correlated with the alpha-complex value, $r(\Delta_{\mathrm{VR}}, \Delta_\alpha) = 0.999$, with the VR value scaled by a constant factor $\langle\Delta_{\mathrm{VR}}\rangle / \langle\Delta_\alpha\rangle = 2.40$. The proportionality $\Delta \propto I(-0.5)$ therefore carries over to VR with a different constant of proportionality $C_{\mathrm{VR}} \approx 2.4\,C_\alpha$. The spectral integral mechanism is not alpha-complex-specific: it captures a property of the spin configuration's correlation structure that is recovered, up to a multiplicative constant, by either filtration. We did not extend the VR comparison to larger $L$ because VR scales poorly with point count.

\subsection{The $\langle m^2 \cdot I(-2\eta)\rangle$ product: a covariance-driven scaling}
\label{sec:hyperscaling}

We define the \emph{topological product} $P(L) = \langle m^2 \cdot I(-2\eta) \rangle$ and examine its scaling. Naively, since $\langle m^2 \rangle \sim L^{-2\beta/\nu}$ and $\langle I(-2\eta) \rangle \sim L^{d+\eta}$, an uncorrelated factorization would give $P \sim L^{d + \eta - 2\beta/\nu}$, which equals $L^d$ when hyperscaling holds ($2\beta/\nu = \eta$ in $d = 2$). One might therefore hope to use $P(L)/L^d$ as a topological hyperscaling diagnostic.

This hope is undermined by a strong negative correlation. By Parseval (Sec.~\ref{sec:decomposition}), $I_0 = N \bar\rho(1-\bar\rho) = N(1-m^2)/4$ for Ising; numerically, $r(m^2, I(-2\eta)) \approx -0.98$. Decomposing the product,
\begin{equation}
\langle m^2 \cdot I \rangle = \langle m^2\rangle \langle I\rangle + \mathrm{Cov}(m^2, I),
\end{equation}
the covariance term is large and \emph{opposite in sign} to the factorized term. This means the observed $P$-slope reflects both the (hypothesized) hyperscaling balance and the covariance correction, and the two are not separable from $P$ alone.

\begin{table}[b]
\caption{$P = \langle m^2 \cdot I(-2\eta)\rangle$ scaling exponents at $\alpha = -2\eta$ (Ising: $-0.5$; Potts $q=3$: $-8/15$; Potts $q=4$: $-1.0$). Errors are bootstrap SEs (1000 resamples of the per-configuration $(m^2, I)$ pairs). Deviation is from the factorized prediction $L^{d+\eta-2\beta/\nu}$.}
\label{tab:hyperscaling}
\begin{ruledtabular}
\begin{tabular}{lccccc}
Model & $\eta$ & $2\beta/\nu$ & Factorized & $P$ slope & Deviation \\
\midrule
2D Ising & 0.250 & 0.250 & 2.000 & $2.22 \pm 0.01$ & $+0.22$ \\
Potts $q\!=\!3$ & 4/15 & 4/15 & 2.000 & $2.02 \pm 0.05$ & $+0.02$ \\
Potts $q\!=\!4$ & 0.500 & 0.250 & 2.250 & $2.58 \pm 0.02$ & $+0.33$ \\
\end{tabular}
\end{ruledtabular}
\end{table}

Table~\ref{tab:hyperscaling} shows the results, reported as deviations from the factorized prediction $L^{d+\eta-2\beta/\nu}$ rather than from $L^d$. For Potts $q = 3$, the $P$-slope $2.02 \pm 0.05$ matches the factorized $2.000$ within $0.4\sigma$; for 2D Ising, the $0.22$ deviation is $19\sigma$ away from factorized and reflects the covariance correction even when hyperscaling holds exactly ($\eta = 2\beta/\nu = 0.25$); for Potts $q = 4$, the additional $0.33$ deviation from the (correctly $L^{2.25}$) factorized prediction is consistent with the multiplicative log corrections that dress every $q = 4$ exponent. The $r \approx -0.98$ correlation between $m^2$ and $I(-2\eta)$ via the shared $I_0$ amplitude (Sec.~\ref{sec:decomposition}) makes this section a covariance-correction analysis rather than an independent hyperscaling diagnostic, and any practitioner wishing to test hyperscaling via PH observables should measure $\langle m^2\rangle$ and $\langle I\rangle$ separately rather than their product.

\section{Discussion}
\label{sec:discussion}

We have derived the topological gap exponent $d + \eta$ from the critical structure factor via the spectral integral mechanism. The derivation rests on three elements: (i)~the empirical proportionality $\Delta \propto I(-2\eta)$ (ratio CV $= 4.1\%$); (ii)~the analytical scaling $I(\alpha) \sim L^{2-\alpha-\eta}$ for IR-dominated integrals; and (iii)~the Parseval proof that the bare PH kernel is flat, establishing the non-perturbative nature of the dressing.

The $I_0/I_{\mathrm{shape}}$ decomposition cleanly separates volume ($L^d$) from anomalous dimension ($L^\eta$) contributions, and explains the per-configuration magnetization-driven variance documented in Sec.~\ref{sec:model}. The critical dimension criterion $d < 2 + \eta$ establishes that this mechanism is fundamentally two-dimensional: in $d \geq 3$, the spectral integral is UV-dominated and a different mechanism (density normalization) is needed.

Two findings deserve further investigation but are not established by the present data. First, the $\alpha$-sweep for Potts $q = 4$ yields $\alpha_{\mathrm{opt}}$ in the interval $[-0.75, -0.5]$ at $L = 32$--$256$, consistent with $-2\eta_{\mathrm{Ising}} = -0.5$ and inconsistent with the model-specific $-2\eta_{q=4} = -1.0$. We emphasize that this is tentative: the $L = 512$ instability documented in Sec.~\ref{sec:alpha_sweep} prevents a clean resolution, and $q = 4$ is the marginal case where logarithmic corrections distort exponents in either direction. \emph{If} the $\alpha_{\mathrm{opt}} \approx -0.5$ result survives confirmation at $L = 1024$ with multi-seed thermalization, it would imply a universal spectral weight determined by the alpha complex geometry rather than by the spin model's anomalous dimension---reinterpreting the $d + \eta$ formula as $d + \eta_{\mathrm{eff}}$ with $\eta_{\mathrm{eff}} \approx 1/4$ a geometric constant of the alpha complex in two dimensions, coincidentally equal to the Ising $\eta$. This is a hypothesis for follow-up work, not a conclusion of the present paper.

Second, the $\langle m^2 \cdot I(-2\eta)\rangle$ product (Sec.~\ref{sec:hyperscaling}) is dominated by the strong negative correlation $r(m^2, I) \approx -0.98$ via the shared $I_0$ amplitude, rather than being an independent hyperscaling diagnostic. After accounting for this covariance, the Potts $q = 3$ exponent matches the factorized prediction within $0.02$, the 2D Ising deviation of $0.22$ is driven by the covariance term itself, and the $q = 4$ deviation reflects the multiplicative log corrections of the marginal case. Hyperscaling tests built from PH observables should therefore use $\langle m^2\rangle$ and $\langle I\rangle$ separately rather than their product.

What remains open is a first-principles derivation of the dressed kernel. Three approaches present themselves: (i)~a Wilsonian RG for the PH functional; (ii)~a topological cluster expansion; or (iii)~a connection to Kac--Rice theory for random field excursion sets~\cite{adler2007}, where Betti numbers depend on spectral moments whose IR divergence at criticality may source the $L^{\eta}$ factor.

\begin{acknowledgments}
Computations used GUDHI~\cite{maria2014} for persistent homology and NumPy/SciPy for numerical analysis. This work was conducted with assistance from Claude (Anthropic).
\end{acknowledgments}

\bibliography{spectral_integral}

\end{document}